\documentclass[11pt]{article}
\usepackage{gtml2025_workshop}

\usepackage{url} % 
\usepackage[utf8]{inputenc} % allow utf-8 input
\usepackage[T1]{fontenc}    % use 8-bit T1 fonts
\usepackage{hyperref}       % hyperlinks
\usepackage{url}            % simple URL typesetting
\usepackage{booktabs}       % professional-quality tables
\usepackage{amsfonts}       % blackboard math symbols
\usepackage{nicefrac}       % compact symbols for 1/2, etc.
\usepackage{microtype}      % microtypography
\usepackage{xcolor}         % colors
\usepackage{comment}

\usepackage{natbib}
\usepackage{amsmath}
\usepackage{amssymb}
\usepackage{mathtools}
\usepackage{amsthm}
\usepackage{multirow}

\usepackage{algorithm}
\usepackage{algpseudocode}

\usepackage[capitalize,noabbrev]{cleveref}

\theoremstyle{plain}
\newtheorem{theorem}{Theorem}[section]
\newtheorem{proposition}[theorem]{Proposition}
\newtheorem{lemma}[theorem]{Lemma}

\theoremstyle{definition}
\newtheorem{definition}[theorem]{Definition}
\newtheorem{example}[theorem]{Example}

\theoremstyle{remark}

\usepackage{thm-restate}

\makeatletter
\def\thmt@innercounters{equation,theorem}
\makeatother

\usepackage[colorinlistoftodos]{todonotes}

\usepackage{amsmath,amsfonts,bm}

\def\eqref#1{equation~\ref{#1}}
\def\1{\bm{1}}

\DeclareMathAlphabet{\mathsfit}{\encodingdefault}{\sfdefault}{m}{sl}
\SetMathAlphabet{\mathsfit}{bold}{\encodingdefault}{\sfdefault}{bx}{n}

\newcommand{\E}{\mathbb{E}}

\newcommand{\KL}{D_{\mathrm{KL}}}

\title{The Geometry of Nonlinear Reinforcement Learning}  
\author{Nikola Milosevic \\
Max Planck Institute for Human Cognitive and Brain Sciences, Leipzig\\
Center for Scalable Data Analytics and Artificial Intelligence (ScaDS.AI), Dresden/Leipzig\\
\texttt{nmilosevic@cbs.mpg.de} \\
\And
Nico Scherf \\
Max Planck Institute for Human Cognitive and Brain Sciences, Leipzig\\
Center for Scalable Data Analytics and Artificial Intelligence (ScaDS.AI), Dresden/Leipzig\\
\texttt{nscherf@cbs.mpg.de} \\
}

\date{\today}

\newcommand{\sS}{\mathcal{S}}
\newcommand{\sA}{\mathcal{A}}

\newcommand{\p}{\omega}  % so we can switch the symbol for the stationary distribution

\newcommand{\C}{\mathrm{C}}

\definecolor{cbs1}{HTML}{00A89D}

\begin{document}

\maketitle

\begin{abstract}
Reward maximization, safe exploration, and intrinsic motivation are often studied as separate objectives in reinforcement learning (RL). 
We present a unified geometric framework, that views these goals as instances of a single optimization problem on the space of achievable long-term behavior in an environment. 
Within this framework, classical methods such as policy mirror descent, natural policy gradient, and trust-region algorithms naturally generalize to nonlinear utilities and convex constraints. 
We illustrate how this perspective captures robustness, safety, exploration, and diversity objectives, and outline open challenges at the interface of geometry and deep RL.
\end{abstract}

\section{Introduction}
Classical Reinforcement Learning (RL) is formalized as maximizing the expected cumulative reward in a Markov Decision Process (MDP). While this linear formulation has driven major advances, many real-world problems demand richer objectives: respecting safety constraints~\citep{dai2023saferlhfsafereinforcement}, encouraging exploration~\citep{hazan2019provably}, or balancing multiple goals~\citep{sun2024prompt, kolev2025dual, grillotti2024quality}. 
Such requirements often correspond to \emph{nonlinear} utility functionals or convex constraints on the agent's long-term behavior~\citep{zhang2020variational,zahavy2021reward}.

A natural space to formulate these problems is the manifold of \emph{discounted state-action occupancy measures}, $\Omega$ in Figure~\ref{fig:teaser-figure}. 
In this space, standard reward maximization is a linear program~\citep{Altman1999ConstrainedMD, puterman2014markov}.
Nonlinear utilities and constraints deform this picture into a general nonlinear program, where dynamic programming no longer applies. 
However, when the utility is concave and constraints are convex, the problem remains a so-called \emph{convex MDP}, which can be solved using a combination of online learning with certain reinforcement learning algorithms~\citep{zhang2020variational, zahavy2021reward}.

The challenge, especially in deep RL, is that the mapping from policy parameters $\theta$ to the occupancy measure of a policy $\omega_{\pi_\theta}\in\Omega$ is highly non-convex. Thus the overall optimization landscape is non-convex, regardless of the utility's form, and insights from the online convex optimization literature do not necessarily apply. 
Our key insight is that the form of the utility functional does not change the problem's solvability by deep on-policy actor-critic methods. 
These algorithms already resort to local, iterative updates due to the non-convexity of the policy parameterization. 
Therefore, extending the utility from linear to nonlinear does not fundamentally change the nature of the optimization problem they solve, as long as the functional is differentiable w.r.t the occupancy measure.

\begin{figure}[ht!]
\centering
% Assuming you have this graphic locally
\includegraphics[width=\linewidth]{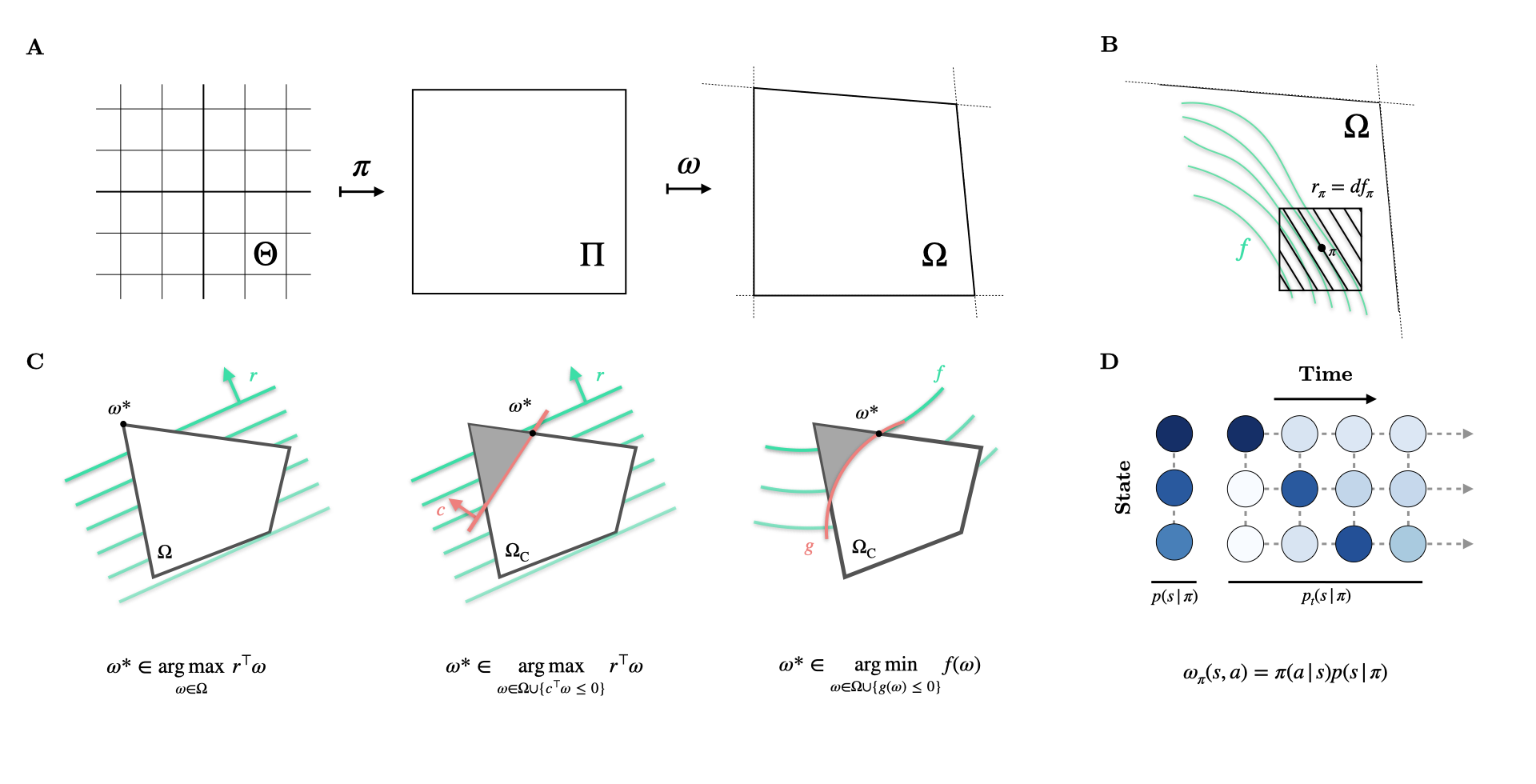}
\caption{Nonlinear MDPs are nonlinear programs on the occupancy space $\Omega$. 
A: The mapping from policy parameters $\theta$ to occupancies $\omega$ is a diffeomorphism. 
B: Nonlinear MDPs with differentiable utilities can be approximated by a linear MDP locally around the current policy $\pi$. 
        The respective reward $r_\pi$ is obtained as the differential of $f$ at $\pi$.
        C: Standard MDPs (left) and constrained MDPs (middle) are linear programs, while convex MDPs (right) are convex programs on $\Omega$.
        D: Intuitively, the occupancy measure can be understood as the probability measure obtained by ``marginalizing out'' the time variable using a geometric distribution. 
        }
\label{fig:teaser-figure}
\end{figure}

In this work, we demonstrate how this "utility-agnostic" view could be formalized geometrically and how it leads to more principled algorithm design for general-utility RL. 
Specifically, we show that standard actor-critic methods are instances of \emph{mirror descent} on the occupancy manifold. 
We then leverage the underlying Hessian geometry of this framework to generalize these methods to the nonlinear constrained case, 
yielding a practical and scalable algorithm for solving general nonlinear MDPs.

\paragraph{Related Work.}
Nonlinear formulations of reinforcement learning (RL) have attracted significant attention as they generalize classical reward maximization and allow for richer objectives and constraints. 
In the unconstrained case with a \emph{concave} utility, the problem can be reformulated as a two-player game around the standard MDP linear program~\citep{zahavy2021reward}. 
Here, a \emph{reward player} iteratively selects reward functions to minimize a payoff, while a \emph{policy player} seeks to maximize it~\citep{Miryoosefi2019convexConstraints, hazan2019provably, nachum2020reinforcement, zhang2020variational, zahavy2021reward, geist2022concaveutilityreinforcementlearning}. 
From the policy player's perspective, this corresponds to an MDP with non-stationary rewards, enabling the use of policy optimization algorithms tailored to such settings, especially methods inspired by mirror descent and proximal methods from online convex optimization~\citep{shani20optimistic, tomar2022mirror, lan2023policy}.
For completeness, \cite{zahavy2021reward} also briefly discuss convex constraints. 
However, to date, theory that treats the general constrained case, as well as solution methods for deep reinforcement learning for nonlinear utilities remain largely unexplored.
Beyond these formulations, several related works further motivate the study of nonlinear RL. 
For example, ~\cite{mutti2023challengingcommonassumptionsconvex, moreno2025onlineepisodicconvexreinforcement, santos2025the} recently examined sample complexity and practical challenges for concave utilities, \cite{gemp2025convex} study convex Markov games, 
and mirror descent has been generalized to regularized and concave-utility RL settings~\citep{geist2022concaveutilityreinforcementlearning, zhan2023policymirrordescentregularized}. 
%The duality between MDPs and probabilistic grahical models has long been recognized~\citep{todorov2008general, ziebart2008maximum, rawlik2012stochastic}, and has some important connection to the framework presented here, see Appendix ....
Finally, our approach is inspired by the Hessian geometry of optimization algorithms~\citep{duistermaat5hessian, alvarez2004hessian, raskutti2014informationgeometrymirrordescent}, which was recently applied to policy optimization~\citep{muller2023geometry, milosevic2024embedding, milosevic2025central}. 

\section{Actor-Critic Methods for Nonlinear Decision Making}

We consider a controlled Markov process $\mathcal{M}=(\sS,\sA,P,\mu,\gamma)$, see Appendix \ref{app:rl-as-linear-programming}. 
For any stationary policy $\pi: s \mapsto p(a)$, the discounted state-action occupancy measure is
\begin{equation}
\omega_\pi(s,a) \coloneqq \pi(a|s)\left[(1-\gamma)\sum_{t=0}^\infty \gamma^t p_{t}(s|\pi)\right],
\end{equation}
where $p_{t}(s|\pi)$ is the probability of state $s$ at time $t$ given policy $\pi$. 
The set of all achievable occupancies forms the \emph{occupancy manifold} $\Omega$, a convex polytope defined by a system of linear equations, see \citep{kallenberg1994survey} and Appendix \ref{occupancy-space}. 
As a subset of the probability simplex, $\Omega$ can be endowed with a (curved) Riemannian structure~\citep{amari1982differential,ay2017information}, which we will discuss below and in Appendix~\ref{app:hessian-geom}.

We define a \emph{Nonlinear MDP} as the following optimization problem on this manifold:
\begin{equation}\label{eq:n-mdp}\tag{N-MDP}
\boldsymbol{\omega^*} \in \arg \max_{\boldsymbol{\omega} \in \Omega}\, \{f(\boldsymbol{\omega}) \mid g_i(\boldsymbol{\omega}) \leq 0, \forall i \},
\end{equation}
where $f$ and $g_i$ are continuously differentiable functions, and each $g_i$ is convex. This framework subsumes standard (constrained) MDPs, where $f, g_i$ are linear~\citep{Altman1999ConstrainedMD}, and convex MDPs~\citep{zahavy2021reward} when $f$ is concave and all $g_i$ are convex. 
Otherwise, local methods are required. This formulation is rich enough to express objectives from imitation learning~\citep{ho2016generative}, exploration~\citep{hazan2019provably}, safety~\citep{achiam2017constrained}, and control-as-inference~\citep{toussaint2006probabilistic, ziebart2008maximum,rawlik2012stochastic}.

\paragraph{Actor-Critic as Mirror Descent}
Our central claim is that modern actor-critic algorithms are fundamentally geometric and utility-agnostic. They can be viewed as mirror descent on the occupancy manifold $\Omega$, implemented entirely in policy space. 
The general mirror descent update step to maximize a utility $f$ on $\Omega$ is:
\begin{equation}\label{eq:pmd}
\omega_{\pi_{k+1}} \in \arg\max_{\omega_\pi \in \Omega} \langle \nabla f(\omega_{\pi_k}),\ \omega_{\pi}-\omega_{\pi_k} \rangle - \eta_k^{-1} D_\phi(\omega_\pi\|\omega_{\pi_k}),
\end{equation}
where the first term is a linear approximation of the change in utility on $\Omega$, and the second term is a Bregman divergence $D_\phi$ that regularizes the size of the update step according to a geometry induced by a potential function $\phi$ on $\Omega$.
Commonly used proximal actor-critic methods like PPO~\citep{schulman2017proximalpolicyoptimizationalgorithms}, TRPO~\citep{schulman2017trust} and more recent improvements~\citep{peng2019advantageweightedregressionsimplescalable, abdolmaleki2018maximum, haarnoja2024learning} solve the surrogate update:
\begin{align}\label{eq:MDPO}
\pi_{k+1} \in \arg \max_{\pi \in \Pi} \mathbb{E}_{s\sim\omega_k, a\sim\pi}[A_{\pi_k}(s,a)] - \eta_k^{-1} \mathbb{E}_{s\sim\omega_k} \mathrm{KL}(\pi_k(\cdot|s)\,\|\,\pi(\cdot|s)),
\end{align}
where the surrogate advantage term, $\mathbb{E}[A_{\pi_k}(s,a)]$, is a first-order approximation of the utility change. Further, the KL regularization term corresponds to the Bregman divergence in \eqref{eq:pmd} generated by the \emph{negative conditional entropy} potential, $\phi(\omega) = \sum_{s,a}\omega(s,a)\log\pi(a|s)$, see~\citep{neu2017unified}.
The equivalence arises because $\mathbb{E}[A_{\pi_k}(s,a)]$ is precisely the utility change, $\langle \nabla f(\omega_{\pi_k}), \omega_{\pi}-\omega_{\pi_k} \rangle$ up to first order in $\pi$ 
if we use an \emph{intrinsic reward} of the form $r_k(s,a) = \partial f(\omega_{\pi_k}/\partial_\omega$, which are the components of $\textup{d}f_{\pi_k}$, the differential of $f$ at $\pi_k$. We refer the reader to Appendix~\ref{app:mirror-descent-equivalence} for a more precise statement.

\section{The Geometry of Nonlinear Actor-Critic Methods}
\begin{figure}[ht!]
    \centering
    \includegraphics[width=\textwidth]{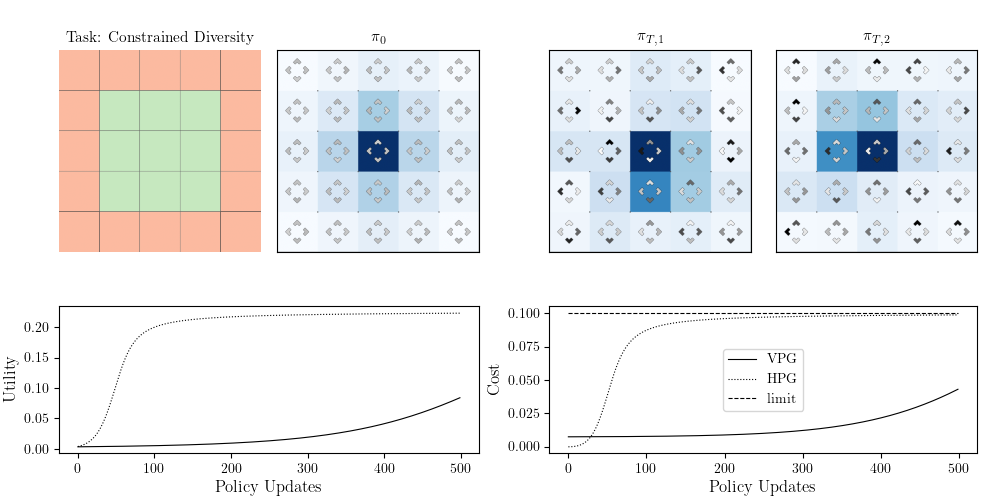}
    \caption{A nonlinear MDP in a 5x5 open gridworld environent. The task is to approximately imitate an initial policy ($\pi_0$) that visits the green squares often and avoids the red squares (top left), implemented as a constraint (0.1 bits in Jensen-Shannon divergence). 
    The imitation should be performed using a diverse mixture of policies (2 in this case).
    The diversity metric is a commonly used convex utility, measuring the mutual information between a binary policy label and the resulting state distribution.
    A close-to-optimal solution is shown on the top right, obtained with a Hessian optimization approach (HPG, see main text), and compared with the vanilla Lagrangian policy gradient (VPG) in the bottom figures.}
    \label{fig:gridworld_cmdp}
\end{figure}
This equivalence reveals the utility-agnostic and inherently geometric nature of actor-critic methods. 
The proximal policy update Eq. \ref{eq:MDPO} only depends on the differential of $f$ at the current $\pi$, not on its global properties.
This can be used to turn nonlinear MDPs into a collection of locally linear MDPs using the differential $\textup{d}f_\pi$ as the linear functional, i.e. using the components of the differential as the reward function entries: $r_\pi(s,a) = (\partial f/\partial_\omega)(s,a)$.
Whether $f$ is linear or nonlinear, proximal actor-critic algorithms perform the same local linearization and regularized update steps in policy parameter space. 
Only Q-Learning does not work out-of-the box, because the Bellman optimality equation does not refer to the optimal policy of the \ref{eq:n-mdp}.
Since the mapping $\theta \mapsto \pi \mapsto \omega_\pi \mapsto f(\omega_\pi)$ is highly non-convex, we are always seeking a local optimum. 
Thus, proximal methods are naturally suited to solve general N-MDPs, or at least as well as they are for solving linear MDPs.

\paragraph{Geometry Matters.} 
While the nonlinearity of the utility function does not change the nature of the policy optimization problem, considering its underlying geometry can indeed improve convergence properties and stability of the updates. 
To achieve stable updates, we must ensure proximity to the previous iterate, so the MDP remains locally linear. 
Because changes in utilities and constraints in the \ref{eq:n-mdp} are directly affected only by the occupancy,
the geometry of occupancy space is what ultimately matters for for proximity.
We argue that the natural geometry to consider is the \emph{Hessian} geometry induced by a particular type of convex potential $\phi: \Omega \to \mathbb{R}\cap\{\infty\}$, see Appendix~\ref{app:hessian-geom}.
We demonstrate in Figure \ref{fig:gridworld_cmdp} the effectiveness of the resulting \emph{Hessian Policy Gradient} (HPG) for a \emph{constrained diversity} problem, see Appendix~\ref{app:implementation-details} for details.

\section{Conclusion and Outlook}
We have outlined a geometric framework for solving nonlinear and constrained MDPs via implicit optimization on occupancy space. This unifies reward maximization, safety, and intrinsic motivation objectives under a common optimization view, enabling principled algorithm design.
A possible future direction is to extend the geometric theory, especially to continuous spaces as in e.g.~\citep{ay2017information, aubin2022mirror},
and to use geometric arguments to derive convergence~\citep{muller2023geometry} and policy improvement guarantees~\citep{schulman2017trust} for parametric policies. 
Another crucial aspect is to understand the practical implications of the nonlinear problem domain for deep function approximation. 
Scalable off-policy, model-based, and offline variants are required to make N-MDP solution tractable and reliable in practice, and respective theories must accompany them. 
Off-policy and offline methods are particularly interesting, as they leverage Q-Learning and experience replay to improve sample efficiency, which requires either understanding how Q-Learning can be properly localized, or by addressing the general existence of Bellman operators, see e.g.~\citep{neu2017unified, geist2019theory}.

\bibliographystyle{plainnat}
\bibliography{gtml2025_workshop}

\appendix

\section{Reinforcement Learning as Linear Programming}\label{app:rl-as-linear-programming}
Every stationary policy $\pi$ in a CMP induces a discounted state-action occupancy measure $\omega_\pi\in\Omega\subset\Delta_{\mathcal S\times\mathcal A}$, which indicates the relative frequencies of visiting a state-action pair, discounted by how far the event lies in the future. We will refer to this measure as the \emph{state-action occupancy} for short, and its marginal in the state variable will be called the \emph{state occupancy}.
\begin{definition}
    The \emph{state-action occupancy measure} is defined as
    \begin{equation}
        \omega_\pi(s)\coloneqq (1-\gamma)
        \sum_{t=0}^\infty \gamma^t \mathbb{P}_\pi (s_t = s)\pi(a|s),
    \end{equation}
    where $\mathbb{P}_\pi (s_t = s)$ is the probability of observing the environment in state $s$ at time $t$ given the agent follows policy $\pi$.
\end{definition}
Note that similar measures can be introduced for the average-reward setting \citep{zahavy2021reward}, and for arbitrary state-action spaces~\citep{laroche23occupancy}. The following property of the occupancy underscores its utility in policy optimization.

\begin{lemma}\label{eq:discounted-expectation}
Given trajectories $\tau=(s_0,a_0,s_1,a_1,...,s_T)$ generated in the CMP $(\sS,\sA,P,\mu,\gamma)$ with policy $\pi$, for a bounded function $f:\sS\times\sA\to\mathbb{R}$ it holds that
    \begin{equation}
        (1-\gamma)\mathbb{E}_{\tau \sim \pi,\mu} \left[ \sum_{t=0}^\infty \gamma^t f(s_t, a_t) \right] = \mathbb{E}_{s, a \sim d^\mu_\pi} [f(s, a)]
    \end{equation}
\end{lemma}

\begin{proof}
    \begin{align}
        (1-\gamma)\mathbb{E}_{\tau \sim \pi,\mu} \left[ \sum_{t=0}^\infty \gamma^t f(s_t, a_t) \right] &= (1-\gamma)\sum_{t=0}^\infty \gamma^t \mathbb{E}_{s_t, a_t \sim \pi,\mu} [f(s_t, a_t)]\\
        &= (1-\gamma)\sum_{t=0}^\infty \gamma^t \sum_{s,a}\mathbb{P}(s_t=s,a_t=a) f(s, a)\\
        &= (1-\gamma)\sum_{t=0}^\infty \gamma^t \sum_{s,a}\mathbb{P}(s_t=s)\pi(a|s) f(s, a)\\
        &= \sum_{s,a}(1-\gamma)\sum_{t=0}^\infty \gamma^t \mathbb{P}(s_t=s)\pi(a|s) f(s, a)\\
        &= \mathbb{E}_{s, a \sim \omega^\mu_\pi} [f(s, a)],
    \end{align}
where we were able to swap the order of the infinite sum and expectation in the first line, since the sum converges uniformly, and $\mathbb{P}(s_t=s,a_t=a)$ is the probability of observing the state $s$ and the action $a$ at time $t$ given the initial distribution $\mu$ and the policy $\pi$.
\end{proof}

\begin{lemma}
The occupancy can also be written as 
\begin{align}
   \omega_\pi(s,a) &= (1-\gamma)\mathbb{E}_{s',a' \sim \omega_\pi} \left[ \delta_{s,a}(s', a') \right]
\end{align}
and
\begin{align}
   \omega_\pi(s,a) &= (1-\gamma)\mathbb{E}_{\tau \sim \pi,\mu} \left[ \sum_{t=0}^\infty \gamma^t \delta_{s,a}(s_t, a_t) \right],
\end{align}
where $\delta_{s,a}$ is the indicator or Dirac distribution at $(s,a)$.
\end{lemma}

The state-action occupancy is useful, since it lets us abstract away the recursive structure of the expected return, and simplify expressions that involve expectations of discounted infinite sums in time. Instead, we can focus on the optimization problems at hand. For finite MDPs, it is well-known that an optimal policy can be found, by first identifying an occupancy that maximizes the expected discounted return
\begin{equation}\label{eq:lin-program}
    \omega^* \in \max_\omega\, \mathbb{E}_{s,a\sim \omega}[r(s,a)]
    \quad \text{subject to } \omega \in\Omega,
\end{equation}
and then conditioning on the state to obtain an optimal policy $\pi^*(s,a) = \omega^*(s,a)/\sum_{a}\omega^*(s,a)$.
Here, $\Omega$ is the set of feasible state-action measures given the CMP parameters~\citep{feinberg2012handbook}.
Next, we want to characterize the state-action space $\Omega$ in more detail, and show that it is a convex polytope, which can be described by a set of linear constraints.

\section{Characterizing state-action space}\label{occupancy-space}
The set of state-action occupancies which are feasible under the CMP dynamics forms a convex polytope in $\mathbb{R}^{|\mathcal{S}||\mathcal{A}|}$ \citep{kallenberg1994survey}, which is characterized as the solution set of the \emph{Bellman flow equations}.

\begin{proposition}[Bellman flow constraints]\label{thm:bellman-flow}
The occupancy of the stationary policy $\pi$ in the CMP with parameters $\mu,\gamma,P$ must be a probability measure over the sample space $\sS\times\sA$, and it must satisfy the Bellman flow equations
\begin{align*}
    \sum_a\omega(s,a) = (1-\gamma) \mu(s) +
    \gamma \sum_{s', a'}\omega(s', a')p(s|s',a'),
\end{align*}
which is a linear system of equations in the variables $\omega(s,a)$, for all $s\in\mathcal{S}$ and $a\in\mathcal{A}$. 
We can write
\begin{equation}\label{eq:bellman-flow}
    \Omega = \{\boldsymbol{\omega}: \boldsymbol{\omega} \in \mathbb{R}^{|\sS||\sA|}_{\geq 0},\ \boldsymbol{B}\boldsymbol{\omega} = \boldsymbol{\mu}\},
\end{equation}
where $\boldsymbol{B}$ is a matrix that depends only on the CMP parameters.
\end{proposition}
\begin{proof}
    Recall that
\begin{align*}
    \omega_\pi(s,a) &= \pi(a|s)\left[(1-\gamma)\sum_{t=0}^\infty \gamma^t \mathbb{P}_\pi (s_t = s)\right],
\end{align*}
which is a probability measure, since $\pi(a|s)$ is a probability measure over actions given the state, $\mathbb{P}_\pi (s_t = s)$ is the probability of observing the state $s$ at time $t$ under the policy $\pi$, and $(1-\gamma)\gamma^t$ is the geometric distribution over time-steps.
This means that we can write $\omega_\pi(s,a) = p(a|s,\pi) \sum_{t=0}^{\infty} p(s|t,\pi) p(t|\gamma)$, which is clearly a probability distribution obtained by marginalization.

Note that, due to the markov property of the CMP, we can write the probability of observing the state $s$ at time $t$ as
\begin{align*}
    \mathbb{P}_\pi (s_t = s) &= \sum_{s',a'}\mathbb{P}_\pi (s_t = s|s_{t-1} = s', a_{t-1}=a')\mathbb{P}_\pi (s_{t-1} = s', a_{t-1}=a')\\
    &= \sum_{s',a'}P(s|s',a')\pi(a'|s')\mathbb{P}_\pi (s_{t-1} = s'),\\
\end{align*}
where $P(s|s',a')$ is the transition kernel of the CMP.
Hence, splitting into the first time-step and the rest of the trajectory, we can write
\begin{align*}
    \omega_\pi(s,a) &= \pi(a|s)\left[(1-\gamma)\mu(s) + (1-\gamma)\sum_{t=1}^\infty \gamma^t \mathbb{P}_\pi (s_t = s),\right]\\
    &= \pi(a|s)\left[(1-\gamma) \mu(s) + \gamma \sum_{s',a'}P(s|s',a')\left[(1-\gamma)\sum_{t'=0}^\infty \gamma^{t'}\mathbb{P}_\pi (s_{t'} = s')\pi(a'|s')\right]\right]\\
    &= \pi(a|s)\left[(1-\gamma)\mu(s) + \gamma \sum_{s',a'}P(s|s',a')\omega(s',a')\right],
\end{align*}
where we re-set the time index to $t'=t-1$ in the second line, and swapped again expectation and uniformly convergent sum.
The Bellman flow equations are obtained by a sum over actions
\begin{align*}
   \sum_a \omega(s,a) &= (1-\gamma)\mu(s) + \gamma \sum_{s',a'}p(s|s',a')\omega(s',a').
\end{align*}

For the final claim, we rewrite the equations in matrix form:
\begin{equation}
    \boldsymbol{B}\boldsymbol{\omega} = \boldsymbol{\mu}.
\end{equation}
Here, $\boldsymbol{B} = (1-\gamma)^{-1}(\boldsymbol{1}_{a}-\gamma\boldsymbol{P})$ is a $|\sS|\times|\sS||\sA|$-matrix with $\boldsymbol{1}_{a}$ being a $|\sS|\times|\sS||\sA|$-matrix with ones along state-action combinations with the same state, and $\boldsymbol{P}$ is the $|\sS|\times|\sS||\sA|$-matrix of the Markov kernel.
\end{proof}
The Bellman flow equations are thus a linear system of equations in the variables $\omega(s,a)$, for all $s\in\mathcal{S}$ and $a\in\mathcal{A}$, which can be solved by standard linear programming techniques, e.g. the simplex method or interior point methods.
The occupancy measure $\omega$ is a probability measure, i.e. $\omega(s,a)\ge 0$ and $\sum_{s,a}\omega(s,a)=1$, which is equivalent to the constraints $\boldsymbol{\omega} \in \Delta(\sS\times\sA)$, where $\Delta(\sS\times\sA)$ is the probability simplex over the state-action pairs.
The Bellman flow equations can be interpreted as a set of constraints on the occupancy measure, which
ensure that the probability mass remains consistent with the transition probabilities. 

Note that we can express the expected discounted return as
\begin{equation}\mathbb{E}_{s, a \sim d^\mu_\pi} [r(s, a)] = \boldsymbol{r}^\top \boldsymbol{\omega}\end{equation}
where $\boldsymbol{d},\boldsymbol{f}\in\mathbb{R}^{|\sS||\sA|}$, up to multiplication of $(1-\gamma)^{-1}$. 
Since the Bellman flow equations and the occupancy matching problem from above are linear in $\boldsymbol{d}$, we can express it as a linear program
\begin{equation}\label{eq:lin-mdp-vector}
    \boldsymbol{\omega^*} \in \arg \max_{\boldsymbol{\omega}}\, \{\boldsymbol{r}^\top \boldsymbol{\omega}
    \mid \boldsymbol{\omega} \in \Delta(\sS\times\sA),\ \boldsymbol{B}\boldsymbol{\omega}=\boldsymbol{\mu}\}.
\end{equation}

A common geometric interpretation, see e.g. ~\citep{ay2017information}, is of $\boldsymbol{r}$ as a covector and $\boldsymbol{\omega}$ as a vector in $\mathbb{R}^{|\mathcal{S}||\mathcal{A}|}$, the expectation being the standard inner product between the function (covector) and the occupancy probability measure (vector). Here, $\boldsymbol{r}$ can also be thought of as a constant one-form acting on the space of occupancies.
In the nonlinear utility case, $df_\pi$ is a nonlinear one-form that we will use to define a \emph{local} or \emph{intrinsic} reward $r_\pi := df_\pi$ which depends on the policy and define a collection of locally linear MDPs on the tangent space $T\Omega$.

\section{Successor Representation} 
The occupancy is related to the well-known \emph{successor representation}~\citep{dayan1993improving} of the policy in an MDP, and it will play a role in the occupancy gradient lemma below.

\begin{definition}[Successor Representation]
The successor representation is defined as
\begin{equation}
    M^\pi(s,a|s',a') \coloneqq \mathbb{E}_\pi \left[\sum_{t=0}^\infty \gamma^t \delta_{s',a'}(s_t,a_t)|s_0=s,a_0=a\right],
\end{equation}
where $\delta_{s',a'}(s_t,a_t)$ is the indicator function that is $1$ if $(s_t,a_t)=(s',a')$ and $0$ otherwise.
\end{definition}

The successor representation and the occupancy measure are equivalent up to the initial state distribution $\mu$, meaning

\begin{equation}
    \omega_{\pi,\mu}(s,a) = (1-\gamma)\sum_{s',a'}M_\pi(s,a|s',a')\pi(a'|s')\mu(s')
\end{equation}
and euqivalently, the successor representation can be expressed as a conditional occupancy measure
\begin{equation}
    M^\pi(s,a|s',a') = (1-\gamma)^{-1}\omega_{\pi,\delta_{s',a'}}(s,a).\
\end{equation}

Intuitively, the successor representation tells us the \emph{discounted} probability of witnessing the event $(s,a)$ conditional on starting at $(s',a')$ and following $\pi$ thereafter. It follows, that the successor representation must also satisfy a Bellman flow equation

\begin{equation}
    \omega(s,a|s',a') = (1-\gamma)\delta_{s',a'}(s,a) + \gamma\sum_{s',a'}\pi(a|s)P(s|s',a')\omega(s',a'|s,a),
\end{equation}
or
\begin{equation}
    M(s,a|s',a') = \delta_{s',a'}(s,a) + \gamma\sum_{s',a'}\pi(a|s)P(s|s',a')M(s',a'|s,a),
\end{equation}

which is the backward Bellman equation for M, see also \citep{touati2021learning}.

One can demonstrate several interesting properties of the SR including that it has the following closed-form matrix-valued expression in finite MDPs
\begin{equation}
    \mathbf M^\pi = [\mathbf I - \gamma \mathbf P^\pi]^{-1}
\end{equation}
where $\mathbf M^\pi$, and $\mathbf P^\pi$ are $|S||A|\times|S||A|$ matrices, and $\mathbf I$ is the indicator matrix of the proper dimensions. 
The inverse exists for $0\leq\gamma < 1$ and equivalently the Neumann series for $\gamma \mathbf P^\pi$ converges to $\mathbf M^\pi$
\begin{equation}
    \mathbf M^\pi = [\mathbf I + \gamma \mathbf P^\pi + (\gamma \mathbf P^\pi)^2+...].
\end{equation}
As the return of a policy can be obtained as an inner product between occupancy and reward, the equivalent relation holds for $M$ and the on-policy Q-Function:
\begin{equation}
    Q_\pi(s,a) = \sum_{s',a'}r(s',a')M_\pi(s',a'|s,a).
\end{equation}

The successor representation features prominently in the study of neural representations~\cite{gershman2018successor}, as well as in transfer learning \citep{barreto2017successor} and intrinsic control algorithms \citep{hansen2019fast} in deep reinforcement learning.
Interestingly, it also appears in the differential of the map between policy and state-action spaces. 
The following result will be relevant in the general version of the policy gradient theorem below, and has been derived as an expression for the jacobian of the map $\pi\mapsto\omega$.

\begin{theorem}[Occupancy Gradient Theorem, Lemma 3.10 in~\citep{muller2024geometry}]
    Let $\pi$ be a stationary policy in the finite CMP $(\sS,\sA,P,\mu,\gamma)$, then the gradient of the occupancy measure w.r.t. the policy parameters $\theta$ is given by
    \begin{equation}\label{eq:occupancy-gradient}
    \nabla_\theta \omega_\pi(s,a) = (1-\gamma)\mathbb{E}_{s',a' \sim \omega} \left[ \nabla_\theta \log\pi_\theta(a'|s')M_\pi(s,a|s',a')\right],
    \end{equation}
    where $\nabla_\theta$ is the gradient w.r.t. the policy parameters $\theta$, and $M_\pi$ is the successor representation given $\pi$.
\end{theorem}

\section{Nonlinear MDPs}\label{app:nonlinear-mdps}
We now extend the above framework to nonlinear MDPs (N-MDPs) with nonlinear utilities and constraints:
\begin{equation}\label{eq:n-mdp-app}
    \boldsymbol{\omega^*} \in \arg \max_{f\boldsymbol{\omega}\in\Omega}\, \{f(\boldsymbol{\omega})
    \mid g(\boldsymbol{\omega})\leq 0\},
\end{equation}
where $\Omega = \{\boldsymbol{\omega} \in \Delta(\sS\times\sA),\ \boldsymbol{B}\boldsymbol{\omega}=\boldsymbol{\mu}\}$ from Eq. \ref{eq:lin-mdp-vector}.
Further, we assume that $f:\Omega \to \mathbb{R}$ is at least once continuously differentiable and $g:\Omega \to \mathbb{R}$ at least twice.
The following result is the basis for the first-order equivalence between occupancy-space mirror descent and proximal actor-critic methods.

\begin{lemma}[Utility Gradient]
    Given an \ref{eq:n-mdp}, the gradient of the nonlinear differentiable utility $f$ w.r.t the policy parameters $\theta$ is given by
    \begin{equation}\label{eq:utility-gradient}
        \nabla_\theta f(\omega_\theta) = \mathbb{E}_{s,a\sim\omega_\pi}[\log \pi_\theta(a|s) A_\pi(s,a)]
    \end{equation}
    where $A_\theta(s,a)$ is the advantage function of the linear MDP with reward functional $r_\pi = \textup{d}f_\pi$.
\end{lemma}

\begin{proof}
    We apply the chain rule to the gradient of $\theta\mapsto\omega_{\pi_\theta}\mapsto f $ as by \citep{zhang2020variational}. 
    Note that the partial derivatives $\partial_\omega f$ are the components of the reward functional from the preceding lemma. 
    Substituting $\nabla_\theta \omega$ using the occupancy gradient theorem from Eq. \ref{eq:occupancy-gradient}, we obtain
    \begin{align}
    \nabla_\theta f(\omega_{\pi_\theta}) 
    &= \sum_{s,a}\frac{\partial f}{\partial \omega(s,a)} \nabla_{\theta} \omega (s,a)\\
    &= (1-\gamma)\mathbb{E}_{s',a' \sim \omega} \left[ \nabla_\theta \log\pi(a'|s')\sum_{s,a} M_\pi(s,a|s',a')r_\pi(s,a)\right]\\
    &= (1-\gamma)\mathbb{E}_{s',a' \sim \omega} \left[ \nabla_\theta \log\pi(a'|s')\sum_{s,a} Q_\pi(s',a')\right],
    \end{align}
    where $Q_\pi$ is the on-policy Q-Function for policy $\pi$ and reward $r_\pi(s,a) = \textup{d}f_\pi$.
    The final result follows from adding the on-policy value function as a state-dependent baseline:
    \begin{align}
    \nabla_\theta f(\omega_{\pi_\theta}) &= (1-\gamma)\mathbb{E}_{s',a' \sim \omega} \left[ \nabla_\theta \log\pi(a'|s')Q_\pi(s',a')\right]\\
    &= (1-\gamma)\mathbb{E}_{s',a' \sim \omega} \left[ \nabla_\theta \log\pi(a'|s')Q_\pi(s',a') - V_\pi(s')\right]\\
    &= (1-\gamma)\mathbb{E}_{s',a' \sim \omega} \left[ \nabla_\theta \log\pi(a'|s')A_\pi(s',a')\right]
\end{align}
\end{proof}

\subsection{Examples of Nonlinear MDPs}
Nonlinear MDPs cover some practically relevant control learning objectives, such as behavioral mutual information, state entropy, divergence-based imitation learning, and control-as-inference objectives, see \citep{zahavy2021reward}.
Some policy learning objectives consider mixtures of two or more policies, and we are going to take a geometric perspective on characterizing them.
Therefore, we consider the Burbeo-Rao divergence, a common symmetrization of the Bregman divergence, and its extension to mixtures of probability distributions~\citep{nielsen2011burbea, burbea2003convexity}.

\begin{definition}[Dispersion of a policy mixture]
    Let us consider a set of policies $\pi_i$ with corresponding state-action occupancies $\omega_i(s,a)$, and let 
%\begin{equation}
%    \bar\pi(a|s) = \frac{\bar\p(s,a)}{\sum_{a'}\bard_{\pi_i}(s,a')} = \frac{\sum_{i}z_id_{\pi_i}(s,a)}{\sum_{j}z_j\sum_{a'}d_{\pi_j}(s,a)}
%\end{equation}
\begin{equation}
    \omega(s,a) = \sum_{i} z_i \omega_i(s,a), \text{ where } z_i \geq 0, \sum_{i} z_i = 1,
\end{equation}
be a finite mixture of occupancies, see e.g.~\citep{peng2019advantageweightedregressionsimplescalable, laroche23occupancy}, and let $\phi$ be a Legendre-type function on $\Omega$. 
The \emph{$\phi$-dispersion} of a policy mixture $\bar{\boldsymbol{\pi}} = \{\pi_1, \dots, \pi_N\}$ is defined as
\begin{equation}
    d_\phi(\bar{\boldsymbol{\pi}}) = \sum_i z_i \phi(\omega_i) - \phi(\omega),
\end{equation}
i.e. the Burbea-Rao divergence of $\omega$ generated by $\phi$.
\end{definition}

The dispersion can be used to compactly characterize popular reward-free objectives like DIAYN \citep{gregor2016variational, eysenbach2018diversity}, and GAIL \citep{ho2016generative}. 

It can be shown that decision problems with objectives of this form correspond to convex MDPs~\citep{zahavy2021reward}, which are known to be solvable using policy optimization methods. To see this, consider the following properties of the polic dispersion.
\begin{restatable}[Jensen-Bregman]{proposition}{avgBregman}
    The \emph{$\phi$-dispersion} of a policy mixture $\bar{\boldsymbol{\pi}} = \{\pi_1, \dots, \pi_N\}$ with respective occupancy mixture $\bar{\omega} = \sum_i z_i \omega_i$ can be expressed as
    \begin{equation}\label{eq:bregman-jensen}
        d_\phi(\boldsymbol{\bar\pi}) = \mathbb{E}_z[\mathrm{D}_h(\omega_z,\bar\omega)],
    \end{equation}
    which is an average of Bregman divergences of the same generator, and reduces to the mutual information between $z$ and $(s,a)$ for the negative conditional entropy. 
    Further, it is jointly strictly convex in $\omega_z$ for strictly convex $h$.
\end{restatable}

\begin{proof}
    \begin{align}
        \sum_i z_i D_h(\omega_i || \bar\omega) &= \sum_i z_i h(\omega_i) - \sum_i z_i h(\bar\omega) - \sum_i z_i \nabla h(\bar\omega)(\omega_i - \bar\omega) \\
        &= \sum_i z_i h(\omega_i) - \sum_i z_i h(\bar\omega) - \nabla h(\bar\omega)(\bar\omega - \bar\omega) \\
        &= \sum_i z_i h(\omega_i) - h\left(\sum_i z_i\omega_i\right).
    \end{align}
    The second claim follows from the definition of mutual information between $(s,a)$-realizations following $\bar\omega$, and the skill $z$:
    \begin{align}
        I(z;s,a) 
            &= H(\bar\omega)-H(\bar\omega|z) \\
            &= H(\bar\omega)-\sum_i z_i H(\omega_i) \\
            &= \sum_i z_i\sum\omega_i\log\omega_i - \sum\bar\omega\log\bar\omega\\
            &= \sum_i z_i\left(\sum\omega_i\log\omega_i - \sum\omega_i\log\bar\omega\right)\\
            &= \sum_i z_iD_{KL}(\omega_i || \bar\omega),
    \end{align}
    which by the above identity is the MPD of the negative conditional entropy.
    Joint convexity follows from the convexity of $D_{KL}(\cdot || \bar\omega)$ and the last line being a convex combination of convex functions.
\end{proof}

In the context of information geometry, $\bar\omega$ is also known as the Bregman centroid~\citep{nielsen2011burbea}. 
%For specific choices of the mixture distributions, a variety of known RL utilities can be instantiated as variants of the JK divergence, all of which are automatically convex MDPs by Eq. \ref{eq:bregman-jensen}.

In the following, we demonstrate some examples in theory and in a small scale simulation, see Figure~\ref{fig:twostate}.

\begin{figure}
    \includegraphics[width=\textwidth]{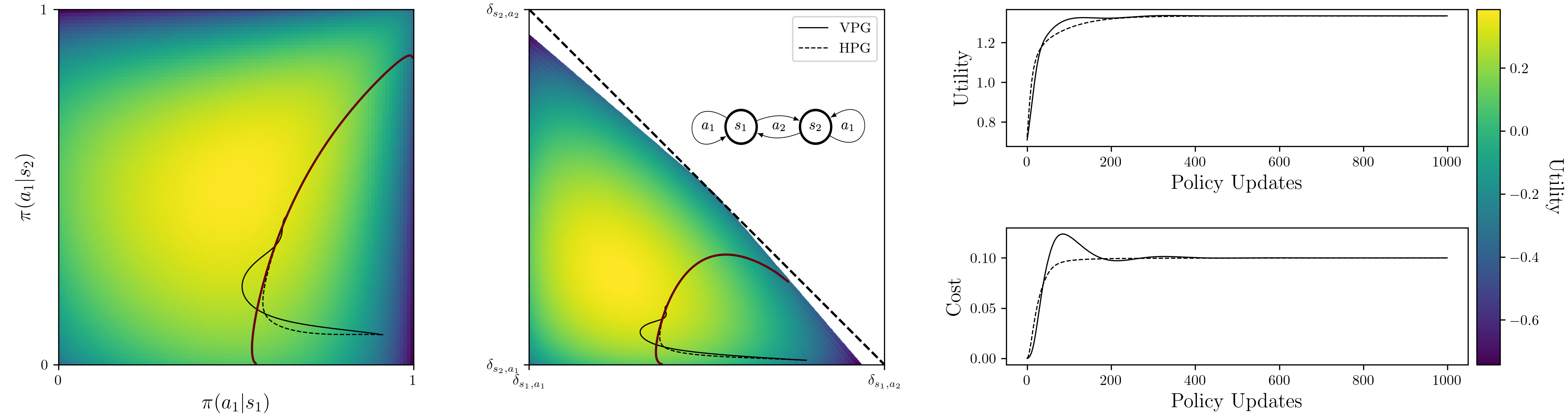}
    \caption{A constrained diversity problem in an MDP with two states and two actions (middle figure, top right), in policy space (left), in linearly projected occupancy space (middle), and a comparison between Lagrangian Vanilla Policy Gradient (VPG) and Hessian Policy Gradient (HPG) in terms of sample efficiency curves.
    The utility functional is discounted state-action entropy in bits.}
    \label{fig:twostate}
\end{figure}

\begin{example}[GAIL~\citep{ho2016generative}]
    Set $N=2$ and $\bar\p(s,a)=\frac{1}{2}\p^\pi(s,a) + \frac{1}{2}\p^E(s,a)$, then 
    \begin{align}
        \max_\pi d_{-H}(\pi, \pi^E),
    \end{align}
    where $\pi^E$ is an expert policy and $-H$ is the negative entropy, is a nonlinear MDP and a \emph{convex} MDP in particular.
    The reward is $r_\pi(s,a) = \sum_i z_i [\log z_i - \log p(i|s, a)]$, which is approximated using an adversarial objective in practice.
\end{example}

\begin{example}[DIAYN~\citep{gregor2016variational, eysenbach2018diversity}]
    Set $z=\operatorname{Unif}(\{1,...,N\})$, then the DIAYN~\citep{eysenbach2018diversity} objective
    \begin{align}
        \max_{\boldsymbol{\bar\pi}} d_{-H_s}(\boldsymbol{\bar\pi}),
    \end{align}
    forms a nonlinear MDP, where $-H_s$ is the negative entropy of the state occupancy $-H_s = -H[\sum_a\omega(\cdot,a)]$.
    The reward is $r_\pi(s,a) = \sum_i z_i [\log p_\pi(i|s) - \log z_i]$. In this case, $p(i|s,\pi)$ is the Bayesian posterior for $\sum_i z_i \omega_{\pi_i}(s,a)$, which is approximated with a variational objective in practice.
\end{example}

\begin{example}[Maximum Entropy Exploration~\citep{hazan2019provably}]
    Set $z=\operatorname{Unif}(\{1,...,N\})$, then the Maximum Entropy Exploration objective
    \begin{align}
        \max_\pi H[\omega^\pi],
    \end{align}
    is a nonlinear MDP, in particular a convex one.
    Here, $H$ is the entropy, and the reward is $r_\pi(s,a) = -(\log \omega_\pi(s,a) + 1)$.
\end{example}

\subsection{Implementation Details}\label{app:implementation-details}
In the computational experiments and implementations we use what we call the \emph{Hessian Policy Gradient} (HPG) updates, a generalization of NPG~\citep{kakade2001natural}, and its constrained counterparts~\citep{milosevic2024embedding} to constrained nonlinear MDPs. 
The HPG update is defined as
\begin{equation}\label{eq:hpg}
\theta_{k+1} = \theta_k + \eta_k \mathbf{H}_\phi^\dagger(\theta_k) \nabla_\theta f(\omega_{\pi_{\theta_k}}),
\end{equation}
where $\mathbf{H}_\phi^\dagger$ is a pseudo-inverse of the Hessian of the potential $\phi$ with respect to the policy parameters $\theta$, and where $\nabla_\theta f(\omega_{\pi_{\theta_k}})$ is estimated as the policy gradient for the intrinsic reward at $\pi$. 
In our small scale experiments, it was sufficient to use out-of-the-box auto-differentiation (pytorch) for computing the Hessian and gradient, and a least-squares solver for the pseudo-inverse.

\paragraph{Example: Constrained Diversity}
In the toy experiments, we demonstrate that policy gradient convergence greatly benefits from Hessian updates, especially in nonlinear constrained settings.
Many practical objectives combine diversity with quality constraints~\citep{kumar2020one, zahavy2021discovering, zahavy2022discovering, sun2024prompt, kolev2025dual, grillotti2024quality} in a constrained policy optimization problem of the form
\begin{equation}
\max_{\boldsymbol\pi} \ \mathrm{Diversity}(\boldsymbol\pi) \ \text{s.t.} \ \mathrm{Quality}(\boldsymbol\pi) \ge \alpha Q^*.
\end{equation}
These metrics are often expressible as $f$ and $g$ in problem ~\ref{eq:n-mdp}, allowing direct application of the Hessian approach. 
In the gridworld experiment in Figure~\ref{fig:gridworld_cmdp}, we chose the DIAYN objective with two policies as the diversity metric and the GAIL objective with a threshold of 0.1 as the quality metric.
In the two dimensional experiments in Figure~\ref{fig:twostate}, we chose the maximum entropy objective as the diversity metric instead.

\subsection{Mirror Descent Equivalence}\label{app:mirror-descent-equivalence}

First, we give a short introduction to Bregman divergences, which are part of the definition of mirror descent. 
For this, we consider a Lagendre potential $\phi$ over a convex subset of Euclidean space $C\subseteq\mathbb R^d$ with a non-empty interior $\operatorname{int}(C)$. 
Then, the \emph{Bregman divergence} induced by $\phi$ is
\begin{align}
    D_\phi(x||y) \coloneqq \phi(x) - \phi(y) - \nabla \phi(y)^\top(x-y),
\end{align}
which is well defined for $x\in C, y\in\operatorname{int}(C)$. 
Intuitively, the Bregman divergence measures the difference between $\phi$ and its linearization at $y$. 
The strict convexity of $\phi$ ensures that $D_\phi(x||y) \ge 0$ and $D_\phi(x||y) = 0$ if and only if $x=y$. 

An important Bregman divergence is the Kullback-Leibler (KL) divergence between finite probability measures $\sum_i p_i=\sum_i q_i=1$
\begin{align}
    \KL(p||q) \coloneqq \sum_{i=1}^d p_i \log \frac{p_i}{q_i}. 
\end{align}

\begin{proposition}
Let the Policy Mirror Descent (PMD) objective be
\begin{equation}
    J_{\textup{PMD}}(\pi) = \langle \nabla f(\omega_k), \omega_\pi - \omega_k \rangle - \frac{1}{\eta} D_\phi(\omega_\pi \| \omega_k)
\end{equation}
and the standard surrogate objective be
\begin{equation}
    J_{\textup{SURR}}(\pi) = \E_{s \sim \omega_k, a \sim \pi}[A_k(s,a)] - \frac{1}{\eta} \E_{s \sim \omega_k} [\KL(\pi_k(\cdot|s) \| \pi(\cdot|s))]
\end{equation}
where the advantage function $A_k$ is computed using the intrinsic reward $r_k := \partial f(\omega_k)/\partial \omega$, and the potential function $\phi$ is the negative conditional entropy, $\phi(\omega) = \sum_{s,a} \omega(s,a) \log \omega(s,a)/\sum_a \omega(s,a)$.
Then, the Taylor expansions of $J_{\textup{PMD}}(\pi_\theta)$ and $J_{\textup{SURR}}(\pi_\theta)$ around the current policy parameters $\theta_k$ are identical up to first order.
Furthermore, the regularizers are identical up to second order.
\end{proposition}

We will show this by explicitly computing the gradients and Hessians of both objectives, evaluated at $\theta = \theta_k$. 
Let $\pi_k = \pi_{\theta_k}$ and $\omega_k = d_{\pi_k}$. We rely on the following properties of the KL-divegence:
\begin{align*}
    \nabla_\theta \KL(\pi_k \| \pi_\theta) |_{\theta_k} &= 0 \\
    \nabla^2_\theta \KL(\pi_k \| \pi_\theta) |_{\theta_k} &= F_s(\theta_k)
\end{align*}
where $F_s(\theta_k)$ is the Fisher Information Matrix (FIM) of the policy $\pi(\cdot|s)$ at state $s$.
In general, the Bregman divergence $D_\phi(\omega_\pi \| \omega_k)$ is minimized (with value 0) at $\pi=\pi_k$. 
Its gradient and Hessian with respect to $\theta$ at $\theta_k$ are:
\begin{align*}
    \nabla_\theta D_\phi(\omega_\pi \| \omega_k) |_{\theta_k} &= 0 \\
    \nabla^2_\theta D_\phi(\omega_\pi \| \omega_k) |_{\theta_k} &= J_\omega(\theta_k)^T H_\phi(\omega_k) J_\omega(\theta_k)
\end{align*}
where $J_\omega(\theta_k) = \nabla_\theta \omega_\pi |_{\theta_k}$ is the Jacobian of $\pi\mapsto\theta$ and $H_\phi(\omega_k) = \nabla_d^2 \phi |_{\omega_k}$ is the Hessian of the potential function. For our chosen potential $\phi$, this product is the FIM weighted by the state-visitation frequencies: $J_\omega^T H_\phi J_\omega = \E_{s \sim \omega_k}[F_s(\theta_k)] = F_k$.

\paragraph{Surrogate Objective}
We compute the gradient of $J_{\textup{SURR}}$ and evaluate it at $\theta_k$:
\begin{align*}
    \nabla_\theta J_{\textup{SURR}}(\pi_\theta) |_{\theta_k} &= \nabla_\theta \E_{s \sim \omega_k, a \sim \pi_\theta}[A_k(s,a)] |_{\theta_k} - \frac{1}{\eta} \nabla_\theta \E_{s \sim \omega_k} [\KL(\pi_k \| \pi_\theta)] |_{\theta_k} \\
    &= \E_{s \sim \omega_k, a \sim \pi_k} [\nabla_\theta \log \pi_\theta(a|s)|_{\theta_k} A_k(s,a)] - \frac{1}{\eta} \cdot 0 \\
    &= \E_{s \sim \omega_k, a \sim \pi_k} [\nabla_\theta \log \pi_\theta(a|s)|_{\theta_k} A_k(s,a)]
\end{align*}
The expectation $\E_{s \sim \omega_k}$ in the KL term is constant with respect to $\theta$ as it uses the fixed distribution from the previous policy.
The Hessian of $J_{\textup{SURR}}$ at $\theta_k$ is
\begin{align*}
    \nabla^2_\theta J_{\textup{SURR}}(\pi_\theta) |_{\theta_k} &= \nabla^2_\theta \E_{s \sim \omega_k, a \sim \pi_\theta}[A_k(s,a)] |_{\theta_k} - \frac{1}{\eta} \nabla^2_\theta \E_{s \sim \omega_k} [\KL(\pi_k \| \pi_\theta)] |_{\theta_k} \\
    &= \nabla^2_\theta \E_{s \sim \omega_k, a \sim \pi_\theta}[A_k(s,a)] |_{\theta_k} - \frac{1}{\eta} \E_{s \sim \omega_k} [F_s(\theta_k)] \\
    &= \E_{s \sim \omega_k, a \sim \pi_k}[\nabla^2_\theta\log \pi(a|s)|_{\theta_k} A_k(s,a)] - \frac{1}{\eta} F_k
\end{align*}

\paragraph{Mirror Descent Objective}
We compute the gradient of $J_{\textup{PMD}}$ and evaluate it at $\theta_k$. 
Note that $\nabla f(\omega_k)$ and $\omega_k$ are constants.
\begin{align*}
    \nabla_\theta J_{\textup{PMD}}(\pi_\theta) |_{\theta_k} &= \nabla_\theta \langle \nabla f(\omega_k), \omega_\pi \rangle |_{\theta_k} - \frac{1}{\eta} \nabla_\theta D_\phi(\omega_\pi \| \omega_k) |_{\theta_k} \\
    &= J_\omega(\theta_k)^T \nabla f(\omega_k) - \frac{1}{\eta} \cdot 0 \\
    &= \nabla_\theta f(\omega_\pi)|_{\theta_k}
\end{align*}
By the Utility Gradient Lemma \ref{eq:utility-gradient}, this is exactly 
$$ \nabla_\theta J_{\textup{PMD}}(\pi_\theta) |_{\theta_k} = \E_{s \sim \omega_k, a \sim \pi_k} [\nabla_\theta \log \pi_\theta(a|s)|_{\theta_k} A_k(s,a)] $$
The gradients match.
The Hessian of $J_{\textup{PMD}}$ at $\theta_k$ is
\begin{align*}
    \nabla^2_\theta J_{\textup{PMD}}(\pi_\theta) |_{\theta_k} &= \nabla^2_\theta \langle \nabla f(\omega_k), \omega_\pi \rangle |_{\theta_k} - \frac{1}{\eta} \nabla^2_\theta D_\phi(\omega_\pi \| \omega_k) |_{\theta_k}
\end{align*}
For the second term, we use the Hessian property of the Bregman divergence:
$$ \nabla^2_\theta D_\phi(\omega_\pi \| \omega_k) |_{\theta_k} = F_k $$
The Hessians of the regularizers also match. 
The Hessian of the first term is related to the Hessian of the surrogate, though not identical. 
However, in the policy optimization literature, one makes the approximation that second order changes in the occupancy are negligible, focusing on the second-order term from the regularizer~\citep{kakade2002Approximately, schulman2017trust}.

\subsubsection{Constrained Policy Geometry}
This equivalence suggests that one can approximate arbitrary mirror descent schemes up to second order, by considering modified versions of TRPO~\citep{schulman2017trust}. 
In C-TRPO~\citep{milosevic2024embedding}, the authors consider mirror functions of the form
\begin{align}
    \Phi_{\operatorname{C}}(\omega) &=\Phi_\textup{K}(\omega) + \Phi_\textup{B}(\omega)\\
                               &\coloneqq\sum_{s,a}\omega(s,a)\log\pi_\omega(a|s)  + \sum_{i=1}^m \beta_i \ell\left(b_i-\sum_{s,a}\omega(s,a)c(s,a)\right), 
\end{align}
where $\omega\in\mathcal{K}_\textup{safe}$ is a feasible state-action occupancy, $\Phi_{\operatorname{K}}$ is the conditional entropy, and $\Phi_\textup{B}$ and $\ell\colon\mathbb R_{>0}\to\mathbb R$ are convex functions with $\ell'(x)\to+\infty$ for $x\searrow0$. 
Possible candidates for $\ell$ are $-\log(x)$ and $x\log(x)$, corresponding to a logarithmic barrier and entropy, respectively. 
This results in an \emph{intractable} policy divergence with Hessian

\begin{align*}
    H_\C(\theta) & = \mathbb{E}_{s\sim \omega_\pi}F(\theta) + \sum_i \beta_i \phi''(b_i-V_{c_i}(\theta)) \nabla_\theta^2 V_{c_i}(\theta) \Big|_{\theta = \theta_k}.
\end{align*}

However, the authors show that the following divergence with the same Hessian can be derived using the common first-order advantage approximation, but for the cost instead of the rewardx:

\begin{align}
    D_\C(\pi||\pi_k) = \bar D_{KL}(\pi||\pi_k) + \beta (\ell(B_k-\mathbb{ A}) - \ell(B_k) - \ell'(B_k)\cdot\mathbb{ A})
\end{align}
where $$\mathbb{A} = \sum_s\omega_{\pi_k}(s)\sum_a \pi(a|s) A^{\pi_k}_c(s,a)$$ is the surrogate cost advantage, and $\bar D_{KL}(\pi||\pi_k)$ is the expected KL-Divergence between the policies w.r.t $\pi_k$'s state occupancy measure. We introduce $B_k = b - V_c(\pi_k)$, i.e. the budget until constraint violation, and focus on a single constraint to reduce notational clutter.

\paragraph{Convex Constraints.} The mirror descent perspective provides a principled way to handle the convex constraints in \eqref{eq:n-mdp}. 
Any Legendre-type potential $\phi$ not only defines a Bregman divergence but also endows the manifold $\Omega$ with a Hessian geometry via its metric tensor. 
This gives rise to a generalized natural policy gradient~\citep{muller2024geometry}. 
Following~\citep{milosevic2024embedding, milosevic2025central}, we can incorporate constraints directly into the geometry by using a \emph{barrier potential}:
\[
b(\omega) = \phi(\omega) + \beta \sum_i \ell(g_i(\omega)),
\]
where $\phi$ is the original potential (e.g., negative conditional entropy) and $\ell$ is a Legendre-type barrier, such as $\ell(x) = -\log(-x)$. Performing natural policy gradient descent with respect to the Hessian metric induced by this new potential $b$ ensures that iterates remain strictly feasible while preserving convergence guarantees~\citep{alvarez2004hessian}.

\section{Hessian and Information Geometries of the Occupancy Space}\label{app:hessian-geom}

The set of all valid state-action occupancy measures for a given CMP forms the natural domain for formulating nonlinear RL problems. 
It can be characterized as
$$\Omega = \{\omega \in \mathbb{R}^{|\sS||\sA|}: \sum_a \omega(s,a) - \gamma \sum_{s',a'} P(s|s',a')\omega(s',a') = (1-\gamma)\mu(s) \quad \forall s \in S \}.$$
Geometrically, $\Omega$ is the intersection of an affine subspace of $\mathcal{R}^{|\sS||\sA|}$ and the probability simplex $\Delta_{\sS\times\sA}$. 
This subspace of $\mathbb{R}^{|\sS||\sA|}$ is defined by the set of probability measures $\omega$, that satisfy the linear Bellman flow equations. 

\paragraph{Fisher-Rao Geometry}

The ambient space, the simplex $\Delta_{\sS\times\sA}$, possesses a canonical and well-studied Riemannian structure known as the Fisher-Rao geometry.
This geometry is induced by the Hessian of the negative joint entropy potential function:
    $$ \phi_{FR}(\omega) = \sum_{s,a} \omega(s,a) \log \omega(s,a) $$
The Hessian, $H_{\phi_{FR}}(\omega)$, defines the Fisher metric on $\Omega$.
The natural parameters corresponding to this potential are the log-probabilities of the occupancy measure itself:
    $$ \eta_{s,a} = \frac{\partial \phi_{FR}}{\partial \omega(s,a)} = \log \omega(s,a) + 1 .$$
The Bregman divergence generated by $\phi_{FR}$ is the Kullback-Leibler (KL) divergence, $D_{KL}(\omega'\|\omega)$.

When this standard geometry is restricted to the manifold $\Omega$, i.e.
$$\eta_{s,a} = \log \sum_{s',a'}M_\pi(s,a|s',a')\pi(a'|s')\mu(s'),$$
it provides a principled way to measure distances between valid occupancy measures, however the resulting exponential family is curved~\cite{amari1982differential}. 
A significant practical challenge is that this geometry requires access to evaluations of the log-occupancy measure, $\log \omega(s,a)$, which are not directly available to a learning agent that only controls its policy, and it is in general difficult to estimate from data.

\paragraph{Kakade Geometry}

In practice, we do not optimize over $\omega$ directly. Instead, we parameterize the occupancy manifold indirectly through a policy $\pi$, as the mapping from a policy to its occupancy measure, $\omega_\pi$, is a diffeomorphism under regularity conditions~\citep{muller2024geometry}. 
This provides an alternative and more practical set of coordinates.
The corresponding potential function is the \textbf{negative conditional entropy}:
$$ \phi_{K}(\omega) = \sum_{s,a} \omega(s,a) \log \pi(a|s) $$
where the policy $\pi(a|s)$ is recovered from the occupancy measure via $\pi(a|s) = \omega(s,a) / \sum_{a'} \omega(s,a')$.
The natural parameters for this geometry are the, still curved, log-probabilities of the policy:
$$ \eta_{s,a} = \frac{\partial \phi_{K}}{\partial \omega(s,a)} = \log \pi(a|s) $$
This is the geometry implicitly used by the Natural Policy Gradient~\citep{kakade2001natural} and its variants.

\paragraph{Motivation for a General Hessian Framework}

The existence of at least two distinct, natural geometries on the same occupancy manifold is a key insight. The Fisher-Rao view is theoretically pure but practically difficult, while the policy-induced view is practical and directly aligned with the learning problem. This enlarges our perspective: rather than committing to a single, fixed geometry, we should consider the choice of geometry as a degree of freedom in algorithm design.
This motivates framing policy optimization within the more general theory of \emph{Hessian structures}~\citep{shima2007geometry}. In this framework, any Legendre-type potential function $\phi$ can be used to define a valid Riemannian structure on $\Omega$. 
By carefully selecting $\phi$, e.g. to incorporate constraints via barrier functions as we did in~\citep{milosevic2024embedding}, we can design novel and more powerful policy update steps that are tailored to the specific structure of the nonlinear MDP we aim to solve.

\end{document}